%% file: TechReport.tex
\documentclass{acm_proc_article-sp}
\newfont{\mycrnotice}{ptmr8t at 7pt}
\newfont{\myconfname}{ptmri8t at 7pt}
\let\crnotice\mycrnotice%
\let\confname\myconfname%

\permission{Permission to make digital or hard copies of all or part of this work for personal or classroom use is granted without fee provided that copies are not made or distributed for profit or commercial advantage and that copies bear this notice and the full citation on the first page. Copyrights for components of this work owned by others than the author(s) must be honored. Abstracting with credit is permitted. To copy otherwise, or republish, to post on servers or to redistribute to lists, requires prior specific permission and/or a fee. Request permissions from permissions@acm.org.}
\conferenceinfo{CIKM'14,}{November 3--7, 2014, Shanghai, China.}
\copyrightetc{Copyright 2014 ACM \the\acmcopyr}
\crdata{978-1-4503-2598-1/14/11\ ...\$15.00.\\
http://dx.doi.org/10.1145/2661829.2662010 
}

\usepackage{array}
\newcolumntype{R}[1]{>{\raggedleft\let\newline\\\arraybackslash\hspace{0pt}}m{#1}}
\usepackage{graphicx}
\input{preamble-stuff}

\newcommand{\ct}{\mathit{ct}}
\usepackage{balance}

\begin{document}

\title{Computing Multi-Relational Sufficient Statistics for  Large Databases}

\numberofauthors{3}
\author{
\alignauthor
Zhensong Qian\\
       \affaddr{School of Computing Science}\\
       \affaddr{Simon Fraser University, CA}\\
       \email{zqian@sfu.ca}
\alignauthor
Oliver Schulte \\
       \affaddr{School of Computing Science}\\
       \affaddr{Simon Fraser University, CA}\\
       \email{oschulte@sfu.ca}
\alignauthor
Yan Sun\\
       \affaddr{School of Computing Science}\\
       \affaddr{Simon Fraser University, CA}\\
       \email{sunyans@sfu.ca}
}
\maketitle  

\begin{abstract} Databases contain information about which relationships do and do not hold among entities. To make this information accessible for statistical analysis requires computing sufficient statistics  that combine information from different database tables. Such statistics may involve any number of {\em positive and negative} relationships. With a naive enumeration approach, computing sufficient statistics for negative relationships is feasible only for small databases. We solve this problem with a new dynamic programming algorithm that performs a virtual join, where the requisite counts are computed without materializing join tables. 
Contingency table algebra is a new extension of relational algebra, that facilitates the efficient implementation of this M\"obius virtual join operation. 
The M\"obius Join scales to large datasets (over 1M tuples) with complex schemas. Empirical evaluation with seven benchmark datasets showed that information about the presence and absence of links can be exploited in feature selection, association rule mining, and Bayesian network learning. 

 \end{abstract}



\section{Introduction} Relational databases contain information about attributes of entities, and which relationships do and do not hold among entities. To make this information accessible for knowledge discovery requires requires computing {\em sufficient statistics}. For discrete data, these sufficient statistics are instantiation counts for conjunctive queries. 
For relational statistical analysis to discover cross-table correlations,  sufficient statistics must combine information from different database tables. This paper describes a new dynamic programming algorithm for computing cross-table sufficient statistics that may contain any number of {\em positive and negative} relationships. Negative relationships concern the nonexistence of a relationship. Our algorithm makes the joint presence/absence of relationships available as features for the statistical analysis of databases. For instance, such statistics are important for learning correlations between different relationship types (e.g., if user $u$ performs a web search for item $i$, is it likely that $u$ watches a video about $i$ ?). 

Whereas sufficient statistics with positive relationships only can be efficiently computed by SQL joins of existing database tables, a table join approach is not feasible for negative relationships. This is because we would have to enumerate all tuples of entities that are {\em not} related (consider the number of user pairs who are {\em not} friends on Facebook). The cost of the enumeration approach is close to materializing the Cartesian cross product of entity sets, which grows exponentially with the number of entity sets involved. It may therefore seem that sufficient statistics with negative relationships can be computed only for small databases. 
We show that on the contrary, assuming that sufficient statistics with positive relationships are available, extending them to negative relationships can be achieved in a highly scalable manner, which does not depend on the size of the database.

\emph{Virtual Join Approach.} Our approach to this problem introduces a new virtual join operation. A virtual join algorithm computes sufficient statistics {\em without} materializing a cross product \cite{Yin2004}. Sufficient statistics can be represented in contingency tables \cite{Moore1998}. Our virtual join operation is a dynamic programming algorithm that successively builds up a large contingency table from smaller ones, {\em without a need to access the original data tables}. We refer to it as the M\"obius Join since it is based on the M\"obius extension theorem \cite{Schulte2014}.

We introduce algebraic operations on contingency tables that generalize standard relational algebra operators. 
We establish a contingency table algebraic identity that reduces the computation of sufficient statistics with $k+1$ negative relationships to the computation of sufficient statistics with only $k$ negative relationships. 
The M\"obius Join applies the identity to construct contingency tables that involve $1,2,\ldots,\ell$ relationships (positive and negative), until we obtain a joint contingency table for all tables in the database. A theoretical upper bound for the number of contingency table operations required by the algorithm is $O(r \log r)$,  where $r$ is the number of sufficient statistics involving negative relationships. In other words, the number of table operations is nearly linear in the size of the required output. 

\emph{Evaluation.} We evaluate the M\"obius Join algorithm by computing contingency tables for seven real-world databases. The observed computation times exhibit the near-linear growth predicted by our theoretical analysis. 
They range from two seconds on the simpler database schemas to just over two hours for the most complex schema with over 1 million tuples from the IMDB database.

Given that computing sufficient statistics for negative relationships is {\em feasible}, the remainder of our experiments evaluate their {\em usefulness}. These sufficient statistics allow statistical analysis to utilize the absence or presence of a relationship as a feature. 
Our benchmark datasets provide evidence that the positive and negative relationship features enhance different types of statistical analysis, as follows. (1) Feature selection: When provided with sufficient statistics for negative and positive relationships,
a standard feature selection method selects relationship features for classification,
(2) Association Rule Mining: A standard association rule learning method includes many association rules with relationship conditions in its top 20 list. 
(3) Bayesian network learning. A Bayesian network provides a graphical summary of the probabilistic dependencies among relationships and attributes in a database. On the two databases with the most complex schemas, enhanced sufficient statistics lead to a clearly superior model (better data fit with fewer parameters). This includes a database that is an order of magnitude larger than the databases for which graphical models  have been learned previously \cite{Schulte2012}. 

\emph{Contributions.} Our main contributions are as follows.
\begin{enumerate}
\item A dynamic program to compute a joint contingency table for sufficient statistics that combine several tables, and that may involve any number of {\em positive and negative }relationships.
\item An extension of relational algebra for contingency tables that supports the dynamic program conceptually and computationally.
\end{enumerate}

We contribute open-source code that implements the M\"obius Join. All code and datasets are available on-line\cite{bib:jbnsite}. Our implementation makes extensive use of RDBMS capabilities. Like the BayesStore system \cite{Wang2008}, our system treats statistical components as first-class citizens in the database. Contingency tables are stored as database tables  in addition to the original data tables. We use SQL queries to construct initial contingency tables and to implement contingency table algebra operations. 

\emph{Paper Organization.} 
We review background for relational databases and statistical concepts. 
The main part of the paper describes the dynamic programming algorithm for computing a joint contingency table for all random variables. 
We define the contingency table algebra. 
A complexity analysis establishes feasible upper bounds on the number of contingency table operations required by the M\"obius Join algorithm. 
We also investigate the scalability of the algorithm empirically. 
The final set of experiments examines how the cached sufficient statistics support the analysis of cross-table dependencies for different learning and data mining tasks.

\section{Background and Notation}

\begin{figure}[htbp] 
 \centering
\resizebox{0.4\textwidth}{!}{
 \includegraphics[width=0.5\textwidth]{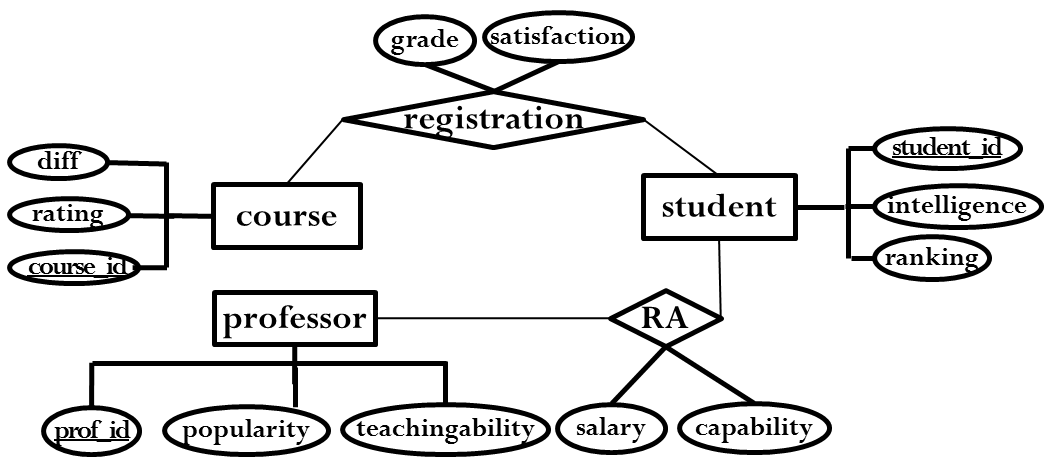}
} 
\caption{A relational ER Design. Registration and RA are many-to-many relationships.}
 \label{fig:university-schema}
\end{figure}
 We assume a standard \textbf{relational schema} containing a set of tables, each with key fields, 
descriptive attributes, and possibly foreign key pointers. 
A \textbf{database instance} specifies the tuples contained in the tables of a given database schema. 
We assume that tables in the relational schema can be divided into {\em entity tables} and {\em relationship tables.} 
This is the case whenever a relational schema is derived from an entity-relationship model (ER model) \cite[Ch.2.2]{Ullman1982}. A \textbf{table join} of two or more tables contains the rows in the Cartesian products of the tables whose values match on common fields.

\subsection{Relational Random Variables} \label{sec:variables}
We adopt function-based notation from logic for combining statistical and relational concepts \cite{Russell2010}.
A domain or \textbf{population} is a set of individuals.
Individuals are denoted by lower case expressions (e.g., $\it{bob}$). 
A \textbf{functor} represents a mapping
$\functor: \population_{1},\ldots,\population_{a} \rightarrow \outdomain_{\functor}$
where $\functor$ is the name of the functor, each $\population_{i}$ is a population, and $\outdomain_{\functor}$ is the output type or \textbf{range} of the functor. 
In this paper we consider only functors with a finite range, disjoint from all populations.  If $\outdomain_{\functor} = \{\true,\false\}$, the functor $\functor$ is a (Boolean) \textbf{predicate}. A predicate with more than one argument is called a \textbf{relationship}; other functors are called \textbf{attributes}. We use uppercase for predicates and lowercase for other functors. Throughout this paper we assume that all relationships are binary, though this is not essential for our algorithm.

A \textbf{(Parametrized) random variable} (PRV) is of the form $\functor(\X_{1},\ldots,\X_{a})$, where each $\X_{i}$ is a first-order variable \cite{Poole2003}. 
Each first-order variable is associated with a population/type. 
\begin{figure}[htbp] 
 \centering
\resizebox{0.45\textwidth}{!}{
 \includegraphics[width=0.5\textwidth]{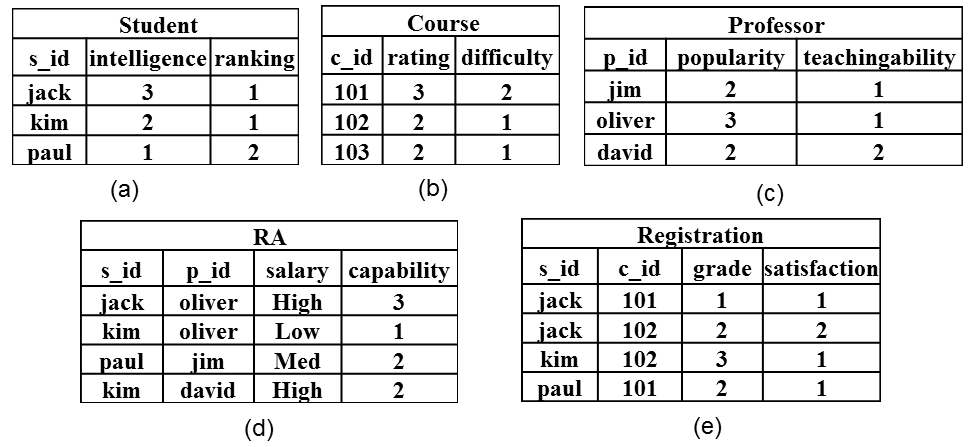} 
}
\caption{Database Instance based on Figure~\ref{fig:university-schema}.  
}
 \label{fig:university-tables}
\end{figure}

The functor formalism is rich enough to represent the constraints of an entity-relationship schema via the following translation: Entity sets correspond to populations, descriptive attributes to functions, relationship tables to relationships, and foreign key constraints to type constraints on the arguments of relationship predicates. Table~\ref{table:translation} illustrates this translation, distinguishing attributes of entities ($\eatts$) and attributes of relationships ($\ratts$). 

\begin{table}[btp] \centering
\resizebox{0.5\textwidth}{!}{
\begin{tabular}[c]
{|l|c|l|l|}\hline
  \begin{tabular}{l}ER \\Diagram \end{tabular}&Type & Functor &Random Variable \\\hline
    \begin{tabular}{l}Relation \\Tables \end{tabular}&RVars &RA & $RA(\P,\S)$ \\\hline
   \begin{tabular}{l}Entity \\Attributes \end{tabular}&$\eatts$ & intelligence, ranking &\begin{tabular}{l} 
  $\{intelligence(\S), ranking(\S)\}$  \\$=\eatts(\S) $\end{tabular} \\\hline
  \begin{tabular}{l} Relationship \\Attributes \end{tabular}&$\ratts$ & capability, salary &\begin{tabular}{l}   $\{capability(\P,\S), salary(\P,\S)\} $ \\= $\ratts(RA(\P,\S))$\end{tabular}\\\hline
   
\end{tabular}
}
\caption{Translation from ER Diagram to 
Random Variables. 
 \label{table:translation}}
\end{table}

\subsection{Contingency Tables}
 
Sufficient statistics can be represented in {\em contingency tables} as follows \cite{Moore1998}. 

Consider a fixed list of  random variables.
A \textbf{query} is a set of $(variable = value)$ pairs where each value is of a valid type for the random variable. 
The \textbf{result set} of a query in a database $\D$ is the set of instantiations of the first-order variables such that the query evaluates as true in $\D$.
For example, in the database of Figure~\ref{fig:university-tables} the result set for the query 
$(\it{intelligence}(\S) = 2$, $\it{rank}(\S) = 1$, $\it{popularity}(\P) = 3$, $\it{teachingability}(\P) = 1$, $\ra(\P,\S) = T)$ is the singleton $\{\langle \it{kim}, \it{oliver}\rangle\}$. 
The \textbf{count} of a query is the cardinality of its result set. 

For every set of variables $\set{V} = \{\V_{1}$,$\ldots,\V_{n} \}$ there is a \textbf{contingency table} $\ct(\set{V})$. 
This is a table with a row for each of the possible assignments of values to the variables in $\set{V}$, and a special integer column called $\qcount$. 
The value of the $\qcount$ column in a row 
corresponding to $V_{1} = v_{1},\ldots,V_{n} = v_{n}$ records the count of the 
corresponding query. 
Figure~\ref{fig:ct} shows the contingency table for the university database. 
The value of a relationship attribute is undefined for entities that are not related.
Following \cite{Russell2010}, 
we indicate this by writing 
$\it{capability(\P,\S)} = n/a $ for a reserved constant $\it{n/a}$. 
The assertion $\it{capability(\P,\S)}$ = n/a is therefore equivalent to the assertion that $\ra(\P,\S) = \false$.
\begin{figure}[htbp]
\begin{center}
\resizebox{0.5\textwidth}{!}{
\includegraphics[width=0.5\textwidth]{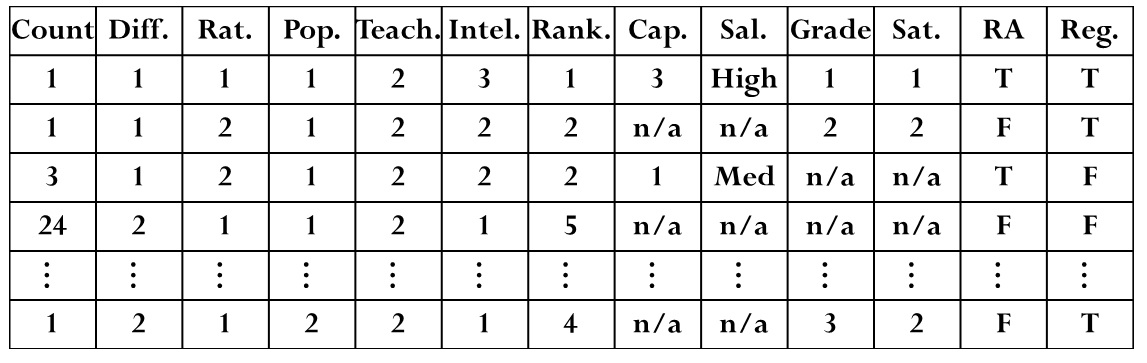}
}
\caption{Excerpt from the joint contingency table for the university database of Figure~\ref{fig:university-tables}. 
\label{fig:ct}}
\end{center}
\end{figure}
A \textbf{conditional contingency table}, written $$\ct(V_{1},\ldots,V_{k}|V_{k+1} = v_{k+1},\ldots, V_{k+m} = v_{k+m})$$
is the contingency table whose column headers are $V_{1},\ldots,V_{k}$ and whose rows comprise the subset that match the conditions to the right of the $\vert$ symbol.  
We assume that contingency tables omit rows with count 0.

\section{Relational  Contingency Tables}
Many relational learning algorithms take an iterative deepening approach: 
explore correlations along a single relationship, then along relationship chains of length 2, 3, etc. 
Chains of relationships form a natural lattice structure, where iterative deepening corresponds to moving from the bottom to the top. 
The M\"obius Join algorithm computes contingency tables by reusing the results for smaller relationships for larger relationship chains.

A relationship variable set is a \textbf{chain} if it can be ordered as a list $[\Relation_{1}(\argterms_{1}),\ldots,\Relation_{k}(\argterms_{k})]$ 
such that each relationship variable $\Relation_{i+1}(\argterms_{i+1})$ shares at least one first-order variable with the preceding terms $\Relation_{1}(\argterms_{1}),\ldots,\Relation_{i}(\argterms_{i})$.
All sets in the lattice are constrained to form a chain.
For instance, in the University schema of Figure~\ref{fig:university-schema}, a 
chain is formed by the two relationship variables
\[\reg(\S,\C),\ra(\P,\S).\]
If relationship variable $\it{Teaches}(\P,\C)$ is added,
we may have a three-element chain \[\reg(\S,\C),\ra(\P,\S),\it{Teaches}(\P,\C).\] 
The subset ordering defines a lattice on relationship sets/chains. 
Figure~\ref{fig:big-lattice} illustrates the  lattice for the relationship variables in the university schema. 
For reasons that we explain below, entity tables are also included in the lattice and linked to relationships that involve the entity in question. 
\begin{figure}[htbp]
\begin{center}
\resizebox{0.5\textwidth}{!}{
\includegraphics[width=0.8\textwidth]{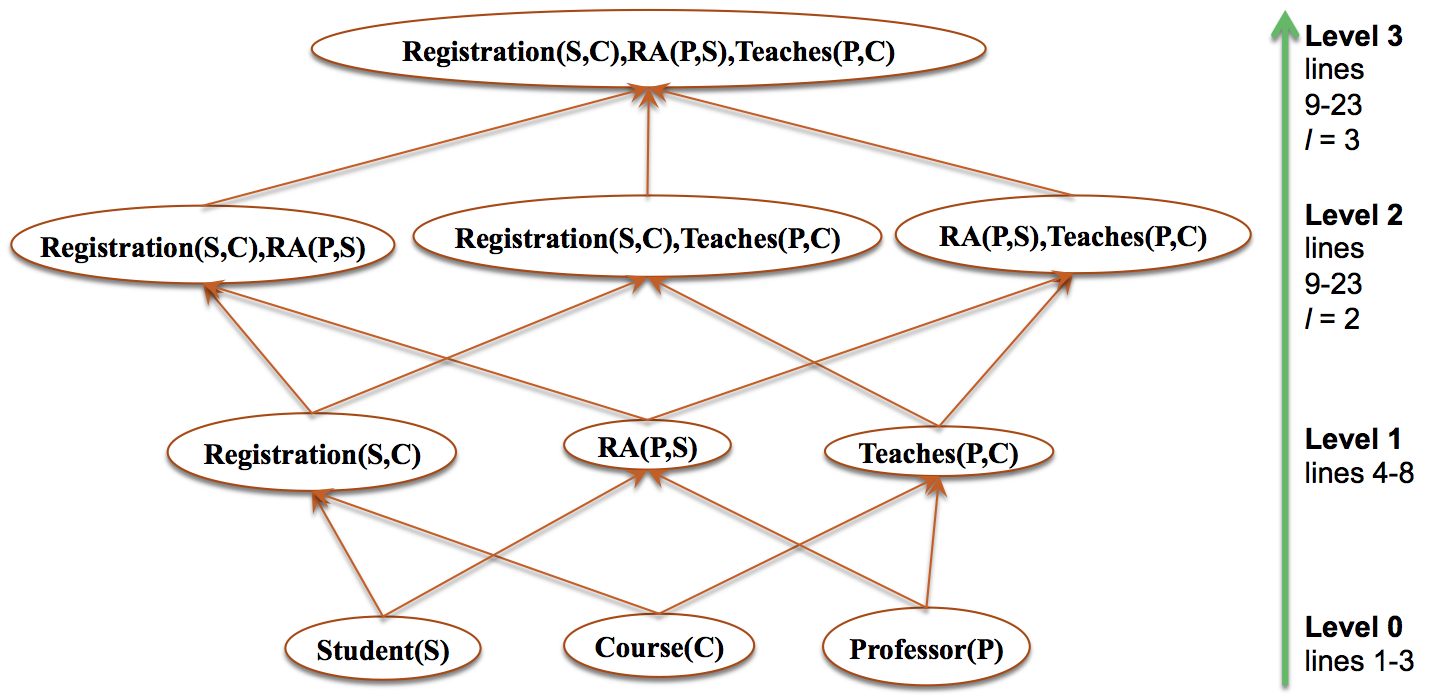}
}

\caption{A lattice of relationship sets for the university schema of Figure~\ref{fig:university-schema}. The M\"obius Join constructs contingency table tables for each relationship chain for each level $\ell$ of the lattice. We reference the lines of the pseudo-code in Algorithm~\ref{alg:fmt}.
\label{fig:big-lattice}}
\end{center}
\end{figure}
With each relationship chain $\set{\Relation}$ (Rchain for short) is associated a $\ct$-table $\ct_{\set{\Relation}}$. 
The variables in the $\ct$-table  $\ct_{\set{\Relation}}$ 
 comprise the relationship variables  in $\set{\Relation}$, and the unary/binary descriptive attributes associated with each of the relationships. To define these, we introduce the following notation (cf. Table~\ref{table:translation}).

\begin{itemize}
\item  $\eatts(\A)$ denotes the attribute variables of a first-order variable $\A$ collectively (1 for unary).
\item $\eatts(\set{R})$ denotes the set of entity attribute variables for the first-order variables that are involved in the relationships in $\set{R}$. 
\item $\ratts(\set{R})$ denotes the set of relationship attribute variables for 
the relationships in $\set{R}$ (2 for binary).
\item $\atts(\set{R}) \equiv \eatts(\set{R}) \cup \ratts(\set{R})$ is the set of all attribute variables in the relationship chain $\set{R}$.
\end{itemize}

In this notation, the variables in the $\ct$-table  $\ct_{\set{\Relation}}$  are denoted as $\set{\Relation} \cup \atts({\set{R}})$. 
The goal of the M\"obius Join algorithm is to compute a contingency table for each chain $\set{\Relation}$. 
In the example of Figure~\ref{fig:big-lattice}, the algorithm computes 10 contingency tables. The $\ct$-table for the top element of the lattice is the \textbf{joint $\ct$-table} for the entire database. 

If a conjunctive query involves only positive relationships, then it can be computed using SQL's count aggregate function applied to a table join. To illustrate, we show the SQL for computing the positive relationship part of the $\ct$-table for the $\ra(\P,\S)$ chain.

\begin{quote}
CREATE TABLE $\ct_{T}$  AS 
\\SELECT Count(*) as  count,  student.ranking, \\student.intelligence, professor.popularity,\\ professor.teachingability, RA.capability, RA.salary  \\
FROM professor, student, RA  \\
WHERE  \\RA.p\_id = professor.p\_id and RA.s\_id = student.s\_id  \\
GROUP BY student.ranking,  student.intelligence, professor.popularity,  professor.teachingability, RA.capability,  RA.salary
\end{quote}

Even more efficient than SQL count queries is the Tuple ID propagation method, a M\"obius Join method for computing query counts with positive relationships only \cite{Yin2004}. 
In the next section we assume that contingency tables for positive relationships only have been computed already, and consider how such tables can be extended to full contingency tables with both positive and negative relationships.

\section{Computing Contingency Tables For Negative Relationships} 

We describe a Virtual Join algorithm that computes the required sufficient statistics without  materializing a cross product of entity sets. 
First, we introduce an  extension of relational algebra that we term \textbf{contingency table algebra}. The purpose of this extension is to 
show that query counts using $k+1$ negative relationships can be computed from two query counts that each involve at most $k$ relationships. 
Second, a dynamic programming algorithm applies the algebraic identify repeatedly to build up a complete contingency table from partial tables.

\subsection{Contingency Table Algebra} \label{sec:cta}
We introduce relational algebra style operations defined on contingency tables.

\subsubsection{Unary Operators} \label{sec:unary}
\begin{description}
\item[Selection] $\sigma_{\selectcond}  \ct$ selects a subset of the rows in the  $\ct$-table  that satisfy condition $\selectcond$. This is the standard relational algebra operation except that the selection condition $\selectcond$ may not involve the $\qcount$ column.
\item[Projection]  
$\project_{\V_{1},\ldots,\V_{k}} \ct$ selects a subset of the  columns in the  $\ct$-table, excluding the count column. 
The counts in the projected subtable are the sum of counts of rows that satisfy the query in the subtable. 
The  $\ct$-table projection  $\project_{\V_{1},\ldots,\V_{k}} \ct$ can be defined by the following SQL code template:
\begin{quote}
SELECT SUM(count) AS count, $V_{1}, \ldots,\ V_{k}$ \\
FROM $\ct$ \\
GROUP BY $V_{1}, \ldots,\ V_{k}$
\end{quote}

\item[Conditioning]  $\condition_{\selectcond}  \ct$ returns a conditional contingency table. Ordering the columns as $(V_{1},\ldots,V_{k}, \ldots,\V_{k+j}$),  suppose that the selection condition is a conjunction of values of the form $$\selectcond = (V_{k+1} = v_{k+1},\ldots, V_{k+j} = v_{k+j}).$$  Conditioning can be defined in terms of selection and projection by the equation:
\begin{equation}
\condition_{\selectcond}  \ct = \project_{\V_{1},\ldots,\V_{k}} (\select_{\selectcond}  \ct) \nonumber
\end{equation}
\end{description}

\subsubsection{Binary Operators} \label{sec:bin}
We use $\set{V}$, $\set{U}$ in SQL templates to denote a list of column names in arbitrary order. The notation $\ct_{1}.\set{V} = \ct_{2}.\set{V}$ indicates an equijoin condition: the contingency tables $\ct_{1}$ and $\ct_{2}$ have the same column set $\set{V}$ and matching columns from the different tables have the same values.
\begin{description}
\item[Cross Product] 
The \textbf{cross-product} of $\ct_{1}(\set{U}),\ct_{2}(\set{V})$ is the Cartesian product of the rows, where the product counts are the products of count. The cross-product can be defined by the following SQL template:

\begin{quote}
SELECT \\($\ct_{1}.\qcount *\ct_{2}.\qcount$) AS $\qcount$,  $\set{U}$, $\set{V}$\\
FROM  $\ct_{1},\ct_{2}$
\end{quote}

\item[Addition] 
 The \textbf{count addition} $\ct_{1}(\set{V}) + \ct_{2}(\set{V})$ adds the counts of matching rows, as in the following SQL template.
\begin{quote}
SELECT 
$\ct_{1}.\qcount$+$\ct_{2}.\qcount$ AS $\qcount$, $\set{V}$ \\
FROM  $\ct_{1},\ct_{2}$\\
WHERE $\ct_{1}.\set{V} = \ct_{2}.\set{V}$
\end{quote}

If a row appears in one $\ct$-table but not the other, we include the row with the count of the table that contains the row. 

\item[Subtraction] %
The \textbf{count difference} $\ct_{1}(\set{V}) - \ct_{2}(\set{V})$ equals $\ct_{1}(\set{V}) + (- \ct_{2}(\set{V}))$ where $- \ct_{2}(\set{V})$ is the same as $\ct_{2}(\set{V})$ but with negative counts. 
Table subtraction is defined only if (i) without the $\qcount$ column, the rows in $\ct_{1}$ are a superset of those in $\ct_{2}$, and (ii) for each row that appears in both tables, the count in $\ct_{1}$ is at least as great as the count in $\ct_{2}$.
\end{description}

\subsubsection{Implementation}\label{sec:imp}

The selection operator can be implemented  using SQL as with standard relational algebra. 
Projection with $\ct$-tables requires use of the GROUP BY construct as shown in Section~\ref{sec:unary}. 

For addition/subtraction, assuming that a sort-merge join is used \cite{Ullman1982}, a standard analysis shows that the cost of a sort-merge join is $\it{size}(table1) + \it{size}(table2) +$ the cost of sorting both tables. 

The cross product is easily implemented in SQL as shown in Section~\ref{sec:bin}. The cross product size is quadratic in the size of the input tables.

\subsection{Lattice Computation of Contingency Tables} \label{sec:mobius}
\begin{figure*}[tb]
\begin{center}
\resizebox{0.8\textwidth}{!}{
\includegraphics[width=0.9\textwidth]{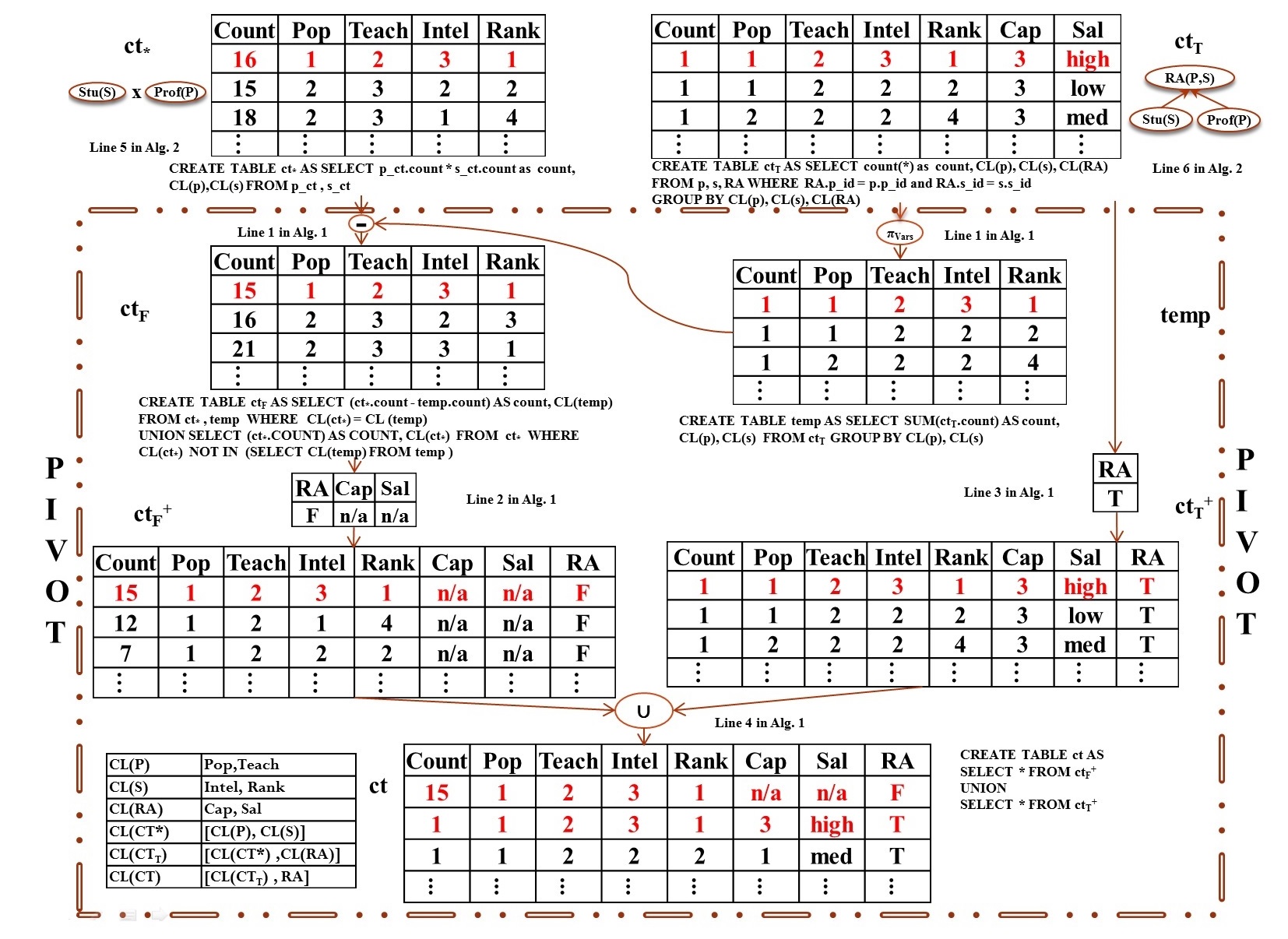}
}

\caption{Top: Equation~\eqref{eq:update} is used to compute the conditional contingency table $\ct_{\false} = \ct(\eatts(R)|R = F)$. (Set $\Nodes = \emptyset$, $R = \ra(\P,\S)$, $\set{R} = \emptyset$). Bottom: 
The Pivot operation computes the contingency table $\ct_{\ra(\P,\S)}$ for the relationship $\ra(\P,\S) := \R_{pivot}$. The $\ct$-table operations are implemented using dynamic SQL queries as shown. Lists of column names are abbreviated as shown and also as follows.
$\ColumnList({\ct_{*}}) = \ColumnList({temp})=\ColumnList({\ct_{F}})$, 
$\ColumnList({\ct}) =  \ColumnList(\ct_{\false}^{+})  = \ColumnList(\ct_{\true}^{+}) $. We reference the corresponding lines of Algorithms~\ref{Pivot_CT_Computation} and~\ref{alg:fmt}.
\label{fig:flow}}
\end{center}
\end{figure*}

This section describes a method for computing the contingency tables level-wise in the relationship chain lattice. We start with a contingency table algebra equivalence that allows us to compute counts for rows with negative relationships from rows with positive relations.
Following \cite{Moore1998}, we use a ``don't care" value $*$ to indicate that a query does not specify the value of a node. For instance, the query $\Relation_{1} = \true, \Relation_{2} = *$ is equivalent to the query $\Relation_{1} = \true$. 

%
\begin{algorithm}[htbp]
\label{Pivot_CT_Computation}
\SetKwData{KwCalls}{Calls}
\SetKwData{KwCondition}{Precondition:}
\KwIn{Two conditional contingency tables   $\ct_{\true} :=\ct(\Nodes,\it{\ratts}(R_{\it{pivot}})|R_{\it{pivot}}=\true$$,\set{R}=\true)$ and  $\ct_{*} :=\ct(\Nodes|R_{\it{pivot}} = *$$,\set{R}=\true)$ .}
\KwCondition  
 The set $\Nodes$ does not contain the relationship variable $R_{\it{pivot}}$ nor any of its descriptive attributes $\ratts(R_{\it{pivot}}$).\;
\KwOut{The conditional contingency table $ \ct(\Nodes,\it{\ratts}(R_{\it{pivot}}),R_{\it{pivot}}|$$\set{R}=\true)$ .}
\begin{algorithmic}[1]
\STATE $\ct_{\false} := \ct_{*} - \pi_{\Nodes}\ct_{\true}$.

\COMMENT{Implements the algebra Equation~\ref{eq:update} in proposition~\ref{PivotCT}.}
\STATE $\ct_{\false}^{+}$ := extend  $\ct_{\false}$ with columns $R_{\it{pivot}}$ everywhere false and $\it{\ratts}(R_{\it{pivot}})$ everywhere $n/a$.
\STATE $\ct_{\true}^{+}$ := extend  $\ct_{\true}$ with columns $R_{\it{pivot}}$ everywhere true.
\STATE \Return $\ct_{\false}^{+} \cup \ct_{\true}^{+}$
\end{algorithmic}
\label{alg:pivot}
\caption{The Pivot function returns a conditional contingency table for a set of attribute variables and all possible values of the relationship $R_{\it{pivot}}$, including $R_{\it{pivot}} = \false$. 
 The set of conditional relationships $\set{R} =(\R_{pivot},\ldots,\R_{\ell})$ 
 may be empty in  which case the Pivot computes an unconditional ct-table. }
\end{algorithm}

\begin{proposition}
\label{PivotCT}
Let $R$ be a relationship variable and let $\set{R}$ be a set of relationship variables. Let $\Nodes$ be a set of variables that 
does not contain $\R$ nor any of the $\ratts$ of $\R$. Let  $\X_{1},\ldots, \X_{l}$ be the first-order variables that appear in $\R$ but not in $\Nodes$, where ${l}$ is possibly zero. Then we have
\begin{flalign}
\label{eq:update}
&\ct(\Nodes \cup \eatts(R)|\set{R} = \true, R = F) = & \\ 
& \ct(\Nodes|\set{R} = \true, R =*) \times \ct(\X_{1}) \times \cdots \times \ct(\X_{l}) \nonumber & \\
& -\ct(\Nodes  \cup \eatts(R)|\set{R} = \true, R = T). \nonumber&
\end{flalign}
If $l = 0$, the equation holds without  the 
cross-product term.
\end{proposition}
\begin{proof}
The equation 
\begin{align}
&\ct(\Nodes  \cup \eatts(R)|\set{R} = \true, R = *) = &\label{eq:update2}  \\ 
&\ct(\Nodes  \cup \eatts(R)|\set{R} = \true, R = T)  + & \nonumber \\ 
&\ct(\Nodes  \cup \eatts(R)|\set{R} = \true, R = F) & \nonumber
\end{align}
holds because the set $\Nodes \cup \eatts(R)$ contains all first-order variables in $R$.\footnote{We assume here that  for each first-order variable, there is at least one $\eatt$, i.e., descriptive attribute.} 
%
 Equation~\eqref{eq:update2} implies
\begin{align} 
&\ct(\Nodes  \cup \eatts(R)|\set{R} = \true, R = \false) =& \label{eq:table-subtract} \\ 
&\ct(\Nodes  \cup \eatts(R)|\set{R} = \true, R = *) -&\nonumber  \\
 & \ct(\Nodes  \cup \eatts(R)|\set{R} = \true, R = \true).& \nonumber
\end{align}

To compute the $\ct$-table conditional on the relationship $\R$ being unspecified, we use the equation
\begin{align}
&\ct(\Nodes  \cup \eatts(R)|\set{R} = \true, R = *) =  &\label{eq:table-multiply}\\
&\ct(\Nodes|\set{R} = \true, R =*) \times \ct(\X_{1}) \times \cdots \times \ct(\X_{l})& \nonumber
\end{align}
which holds because if the set $\Nodes$ does not contain a first-order variable of $\R$, then the counts of the associated $\eatts(\R)$ are independent of the counts for $\Nodes$. 
If $l = 0$, there is no new first-order variable, and Equation~\eqref{eq:table-multiply} holds without  the cross-product term.
Together Equations~\eqref{eq:table-subtract} and~\eqref{eq:table-multiply} establish the proposition.
\end{proof}


Figure~\ref{fig:flow} illustrates Equation~\eqref{eq:update}. 
The construction of the $\ct_{\false}$ table in 
Algorithm~\ref{Pivot_CT_Computation} provides pseudo-code for applying Equation~\eqref{eq:update} to compute a complete $\ct$-table, given a partial table where a specified relationship variable $\Relation$  is true,
and another partial table that does not contain the relationship variable. 
We refer to $\Relation$ as the \textbf{pivot} variable. 
For extra generality, Algorithm~\ref{Pivot_CT_Computation} applies Equation~\eqref{eq:update} with a condition that lists a set of relationship variables fixed to be true.  Figure~\ref{fig:flow} illustrates the  Pivot computation for the case of only one relationship. 
Algorithm~\ref{level-wise-subtract} shows how the Pivot operation can be applied repeatedly to find all contingency tables in the relationship lattice. 

\begin{algorithm*}[tb]
\label{level-wise-subtract}
\SetKwData{KwCalls}{Calls}
\SetKwData{Notation}{Notation}
\KwIn{A relational database $\D$; a set of  variables}
\KwOut{A contingency table that lists the count in the database $D$ for each possible assignment of values to each variable.}
\begin{algorithmic}[1]
\FORALL{first-order variables $\X$}
\STATE compute $\ct(\eatts(\X))$ using SQL queries.
\ENDFOR
\FORALL{relationship variable $\R$}
\STATE $\ct_{*} := \ct(\X) \times \ct(\Y)$ where $\X$,$\Y$ are the first-order variables in $\R$.
\STATE $\ct_{\true} := \ct(\eatts(\R)|\R = \true)$ using SQL joins.
\STATE Call  $\it{Pivot}(\ct_{\true},\ct_{*})$ to compute $\ct(\eatts(\R),\ratts(\R),\R)$.
\ENDFOR
\FOR{Rchain length $\ell=2$ to $m$}
\FORALL{Rchains $\set{\R} = R_{1},\ldots,\R_{\ell}$}
\STATE $Current\_\ct :=  \ct(\eatts(\R_{1},\ldots,\R_{\ell}),\ratts(\R_{1},\ldots,\R_{\ell})|\R_{1}=\true,\ldots,\R_{\ell}=\true)$ using SQL joins.
\FOR{$i=1$ to $\ell$} \label {reﬂine:innerloop}
\IF{ $i$ equals  1}
\STATE $\ct_{*} := \ct(\eatts(\R_{2},\ldots,\R_{\ell}),\ratts(\R_{2},\ldots,\R_{\ell})|
\R_{1}=*,\R_{2} = \true,\ldots,\R_{\ell}=\true) \times \ct(\X)$ where $\X$ is the first-order variable in $\R_{1}$, if any, that does not appear in $\R_{2},\ldots,\R_{\ell}$
\COMMENT{$\ct_{*}$ can be computed from a $\ct$-table for a Rchain of length $\ell-1$.}
\ELSE
\STATE $\eatts_{\bar{i}} := \eatts(\R_{1},\ldots,\R_{i-1},\R_{i+1},\ldots,\R_{\ell})$.
\STATE $\ratts_{\bar{i}} := \ratts(\R_{1},\ldots,\R_{i-1},\R_{i+1},\ldots,\R_{\ell})$.
\STATE $\ct_{*} := \ct(\eatts_{\bar{i}}, \ratts_{\bar{i}},\R_{1},\ldots,\R_{i-1})|
\R_{i}=*,\R_{i+1} = \true,\ldots,\R_{\ell}=\true) \times \ct(\Y)$ where $\Y$ is the first-order variable in $\R_{i}$, if any, that does not appear in $\set{\R}$. 
\ENDIF \\
\STATE $Current\_\ct :=  \it{Pivot}(Current\_\ct,\ct_{*})$.
\ENDFOR 
\COMMENT{Loop Invariant: After  iteration $i$, the table $Current\_\ct$ equals 
$\ct(\eatts(\R_{1},\ldots,\R_{\ell}), \ratts(\R_{1},\ldots,\R_{\ell}),\R_{1},\ldots,\R_{i}|\R_{i+1} = \true,\ldots,\R_{\ell}=\true)$}
\ENDFOR
\COMMENT{Loop Invariant: The $\ct$-tables for all Rchains of length $\ell$ have been computed.}
\ENDFOR 
\STATE \Return the $\ct$-table for the Rchain involves all the relationship variables.
\end{algorithmic}
\label{alg:fmt}
\caption{M\"obius Join algorithm for Computing the Contingency Table for Input Database}
\end{algorithm*}

{\em Initialization.} Compute $\ct$-tables for entity tables.
Compute $\ct$-tables for each single relationship variable $\Relation$ , conditional on $\Relation = \true$. 
If $\Relation = \ast$, then no link is specified between the first-order variables involved in the relation $\Relation$. Therefore the individual counts for each first-order variable are independent of each other and the joint counts can be obtained by the cross product operation. 
Apply the Pivot function to construct the  complete $\ct$-table for relationship variable $\Relation$. 

{\em Lattice Computation.} The goal is to compute $\ct$-tables for all relationship chains of length $>1$. For each relationship chain, order the relationship variables in the chain arbitrarily. Make each relationship variable in order the Pivot variable $\Relation_{i}$. For the current Pivot variable $\Relation_{i}$, find the conditional $\ct$-table where $\Relation_{i}$ is unspecified, and the subsequent relations $\Relation_{j}$ with $j>i$ are true. This $\ct$-table can be computed from a $ct$-table for a shorter chain that has been constructed already. The conditional $ct$-table   has been constructed already, where $\Relation_{i}$ is true, and the subsequent relations are true (see loop invariant). Apply the Pivot function to construct the  complete $\ct$-table, for any Pivot variable $\Relation_{i}$,  conditional on the subsequent relations being true. 
\begin{figure}[htbp]
\begin{center}
\resizebox{0.5\textwidth}{!}{
\includegraphics{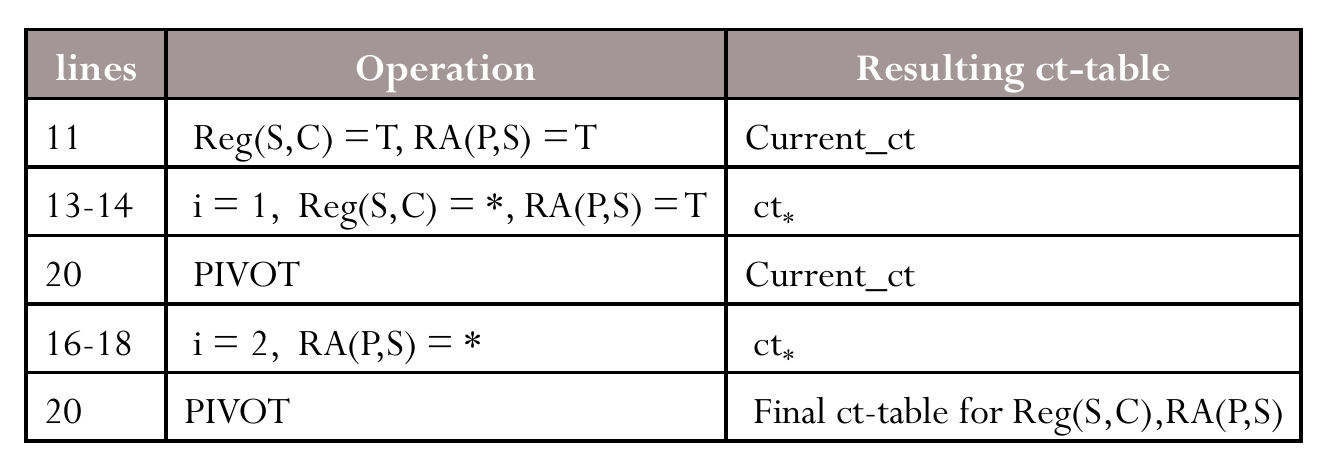}}
\caption{Illustrates the relationship chain loop of Algorithm~\ref{alg:fmt} (lines 11-21) for the chain $\set{\R}= \it{Reg}(\S,\C),\ra(\P,\S)$. This loop is executed for each relationship chain at each level.
\label{fig:rchain-loop}}
\end{center}
\end{figure}

\subsection{Complexity Analysis} 
\label{sec:complexity} 

The key point about the M\"obius Join (\MJ) algorithm 
is that it avoids materializing the cross product of entity tuples. {\em The algorithm accesses  only \textbf{existing} tuples, never constructs nonexisting tuples.} The number of $\ct$-table operation is therefore independent of the number of data records in the original database. We bound the total number of $\ct$-algebra operations performed by the M\"obius Join algorithm in terms of the size of its output: the number of sufficient statistics that involve negative relationships. 
\begin{proposition}
The number of $\ct$-table operations performed by the M\"obius Join algorithm is bounded as $$\it{\# \ct\_\it{ops}} = O(r \cdot \log_{2} r)$$ where $\row$ is the number of sufficient statistics that involve negative relationships.
\end{proposition}

To  analyze the computational cost, we examine the total number of $\ct$-algebra operations performed by the  M\"obius Join algorithm. 
%
We provide  upper bounds in terms of two parameters: the number of relationship nodes  $m$, and  the number of rows $\row$  in the $\ct$-table that involve negative relationships.
For these parameters we establish that
$$\it{\# \ct\_\it{ops}} = O(r \cdot \log_{2} r) = O(m \cdot 2^{m}) .$$
This shows the efficiency of our algorithm for the following reasons.
(i) Since the time cost of any algorithm must be at least as great as the time for writing the output, which is as least as great as $\row$, 
 the  M\"obius Join algorithm adds at most a logarithmic factor to this lower bound. 
(ii)  The second upper bound means that the number of $\ct$-algebra operations is fixed-parameter tractable with respect to $m$.\footnote{For arbitrary $m$, the problem of computing a $\ct$ table in a relational structure is \#P-complete \cite[Prop.12.4]{Domingos2007}.} In practice the number $m$ is on the order of the number of tables in the database, 
which is very small compared to the number of tuples in the tables.

\emph{Derivation of Upper Bounds.}
For a given relationship chain of length $\ell$,  the  M\"obius Join algorithm goes through the chain linearly (Algorithm~\ref{alg:fmt} 
inner for loop line~\ref{reﬂine:innerloop}
). 
At each iteration, it computes a $\ct_{*}$ table with a single cross product, then performs a single Pivot operation.
 Each Pivot operation requires three  $\ct$-algebra operations. 
Thus overall, the number of  $\ct$-algebra operations for a relationship chain of length $\ell$ is $6 \cdot \ell = O(\ell)$. For a fixed length $\ell$, there are at most $\binom{m}{\ell}$ relationship chains. Using the known identity\footnote{math.wikia.com/wiki/Binomial\_coefficient, Equation 6a}
\begin{align} 
\sum_{\ell=1}^{m} {m\choose \ell} \cdot \ell = m \cdot  2^{m-1} \label{eq:upperbound}
\end{align}
we obtain the $O(m \cdot  2^{m-1}) = O(m \cdot  2^{m})$ upper bound.

For the upper bound in terms of $\ct$-table rows $\row$, we note that the output $\ct$-table can be decomposed into $2^{m}$ subtables, one for each assignment of values to the $m$ relationship nodes. 
Each of these subtables contains the same number of rows $d$ , one for each possible assignment of values to the attribute nodes. 
Thus the total number of rows is given by $r = d \cdot 2^m.$ 
Therefore we have 
$m \cdot 2^{m} = \log_{2} (r/d) \cdot r/d \leq \log_{2}(r) \cdot r.$
Thus the total number of $\ct$-algebra operations is $O(r \cdot \log_{2}(r))$.

From this analysis we see that both upper bounds are overestimates. (1) Because relationship chains must be linked by foreign key constraints, the number of valid relationship chains of length $\ell$ is usually much smaller than the number of all possible subsets ${m\choose \ell}$. (2) The constant factor $d$ grows exponentially with the number of attribute nodes, so $\log_{2}(r) \cdot r$ is a loose upper bound on $\log_{2} (r/d) \cdot r/d$. 
We conclude that the number of $\ct$-algebra operations is not the critical factor for scalability, but rather the cost of 
carrying out a single $\ct$-algebra operation.

This means that
if the number $\row$ of sufficient statistics is a feasible bound on computational time and space, then computing the sufficient statistics is feasible. In our benchmark datasets, the number of sufficient statistics was feasible, as we report below. 
In Section~\ref{sec:conclusion} below we discuss options in case the number of sufficient statistics  grows too large.
%
%
%

\section{Evaluation of Contingency Table Computation} 
We describe the system and the datasets we used.
Code was written in Java, JRE 1.7.0.  and executed with 8GB of RAM and a single Intel Core 2 QUAD Processor Q6700 with a clock speed of 2.66GHz (no hyper-threading). The operating system was Linux Centos 2.6.32. 
The MySQL Server version 5.5.34 was run with 8GB of RAM and a single core processor of 2.2GHz. 
All code and datasets are available on-line~\cite{bib:jbnsite}.

\begin{table}[hbtp] \centering
\resizebox{0.5\textwidth}{!}{
\begin{tabular}[c]
{|l|c|c|r|c|}\hline
 \textbf{Dataset} & \textbf{\begin{tabular}[l] {ll} \#Relationship \\Tables/ Total \end {tabular}} & \textbf{\begin{tabular}[l] {ll} \#Self \\Relationships\end {tabular}}  & \textbf{\#Tuples} & \textbf{\#Attributes}  \\\hline
    Movielens &1 / 3 & 0  & 1,010,051 & 7\\\hline
    Mutagenesis & 2 / 4 & 0 & 14,540 & 11\\\hline
    Financial &3 / 7 & 0  &  225,932& 15\\\hline
   Hepatitis &3 / 7 & 0 &12,927  & 19\\\hline
   IMDB &3 / 7 & 0 &1,354,134  & 17\\\hline
    Mondial &2 / 4 & \textbf{1} &  870& 18\\\hline
    UW-CSE &2 / 4 & \textbf{2}  & 712 & 14\\\hline   
\end{tabular}
}
\caption{Datasets characteristics. \#Tuples = total number of tuples over all tables in the dataset. 
  \label{table:datasetsize}}
\end{table}

\subsection{Datasets}

We used seven benchmark real-world databases. For detailed descriptions and  the sources of the databases, please see reference~\cite{Schulte2012}. Table~\ref{table:datasetsize} summarizes basic information about the benchmark datasets.  A  self-relationship 
relates two entities of the same type (e.g. $\it{Borders}$ relates two countries in Mondial). Random variables for each database were defined as described in Section~\ref{sec:variables} (see also \cite{Schulte2012}). IMDB is the largest dataset in terms of number of total tuples (more than 1.3M tuples) and schema complexity. 
It combines the MovieLens database\footnote{www.grouplens.org, 1M version} with data from the Internet Movie Database (IMDB)\footnote{www.imdb.com, July 2013} following \cite{Peralta2007}.

\subsection{Contingency Tables With Negative Relationships: Cross Product vs. M\"obius Join}

In this subsection we compare two different approaches for constructing the joint contingency tables for all variables together, for each database: Our M\"obius Join algorithm (MJ) vs. materializing the cross product (CP) of the entity tables for each first-order variable (primary keys).
Cross-checking the MJ contingency tables with the cross-product contingency tables confirmed the correctness of our implementation. Table~\ref{table:cttimes} compares the time and space costs of the MJ vs. the CP approach. The cross product was materialized using an SQL query. 
The ratio of the cross product size to the number of statistics in the $\ct$-table measures how much compression the $\ct$-table provides compared to enumerating the cross product. 
It shows that cross product materialization  requires an infeasible amount of space resources.
The $\ct$-table provides a substantial compression of the statistical information in the database, by a factor of over 4,500 for the largest database IMDB.  

\begin{table}[htbp] \centering
\resizebox{0.5\textwidth}{!}{
\begin{tabular}{|l|r|r|r|r|r|}\hline 
 \textbf{Dataset} & \textbf{MJ-time}(s) & $\textbf{CP-time}(s)$& \textbf{CP-\#tuples}  & \textbf{\#Statistics} & \textbf{\begin{tabular}{l}Compress \\Ratio
 \end{tabular}} \\\hline
Movielens &2.70&703.99 &23M &252 &93,053.32\\\hline
Mutagenesis &1.67&1096.00 & 1M &1,631 &555.00  \\\hline
Financial &  1421.87&N.T. &149,046,585M &3,013,011 &49,467,653.90   \\\hline
Hepatitis &3536.76&N.T. &17,846M& 12,374,892 &1,442.19 \\\hline
IMDB &7467.85&N.T. &5,030,412,758M&15,538,430 & 323,740,092.05 \\\hline
Mondial &1112.84&132.13&5M&1,746,870&2.67  \\\hline
UW-CSE &3.84&350.30& 10M&2,828 & 3,607.32\\\hline

\end{tabular}
}
\caption{Constructing the contingency table for each dataset. 
M = million. N.T. = non-termination. Compress Ratio = CP-\#tuples/\#Statistics.
  \label{table:cttimes}}
\end{table}

{\em Computation Time.} The numbers shown are the complete computation time for all statistics. For faster processing, both methods used a B+tree index built on each column in the original dataset. The \MJ method also utilized B+ indexes on the $\ct$-tables. We include the cost of building these indexes in the reported time. 
%
%
The M\"obius Join algorithm returned a contingency table with negative relationships in feasible time. On the biggest dataset IMDB with 1.3 million tuples, it took just over 2 hours. 

The cross product construction did not always terminate, crashing after around 4, 5, and 10 hours on Financial, IMDB and Hepatitis respectively. When it did terminate, it took orders of magnitude longer than the \MJ ~method except for the Mondial dataset. Generally the higher the compression ratio, the higher the time savings. On Mondial the compression ratio is unusually low, so materializing the cross-product was faster. 
\begin{table}[htbp]
  \centering
\resizebox{0.5\textwidth}{!}{
\begin{tabular}{|l|r|r|r|r|}\hline
Dataset & \multicolumn{1}{r|}{Link On} & \multicolumn{1}{c|}{Link Off} & \multicolumn{1}{c|}{\#extra  statistics} & \multicolumn{1}{c|}{extra time (s)}    \\ \hline
MovieLens & 252   & 210   & 42    & 0.27    \\ \hline
Mutagenesis & 1,631 & 565   & 1,066 & 0.99    \\ \hline
Financial & 3,013,011 & 8,733 & 3,004,278 & 1416.21    \\ \hline
Hepatitis & 12,374,892 & 2,487 & 12,372,405 & 3535.51    \\ \hline
IMDB  & 15,538,430 & 1,098,132 & 14,440,298 & 4538.62    \\ \hline
Mondial & 1,746,870 & 0     & 1,746,870 & 1112.31    \\ \hline
UW-CSE & 2,828 & 2     & 2,826 & 3.41    \\ \hline
\end{tabular}%

}
  \caption{Number of Sufficient Statistics for Link Analysis On and Off. Extra Time refers to the total \MJ time (Table~\ref{table:cttimes} Col.2) minus the time for computing the positive statistics only.}
  \label{table:link-onoff}%
\end{table}%

\begin{figure}[htbp]
\begin{center}
\resizebox{0.4\textwidth}{!}{
\includegraphics[width=0.6\textwidth]{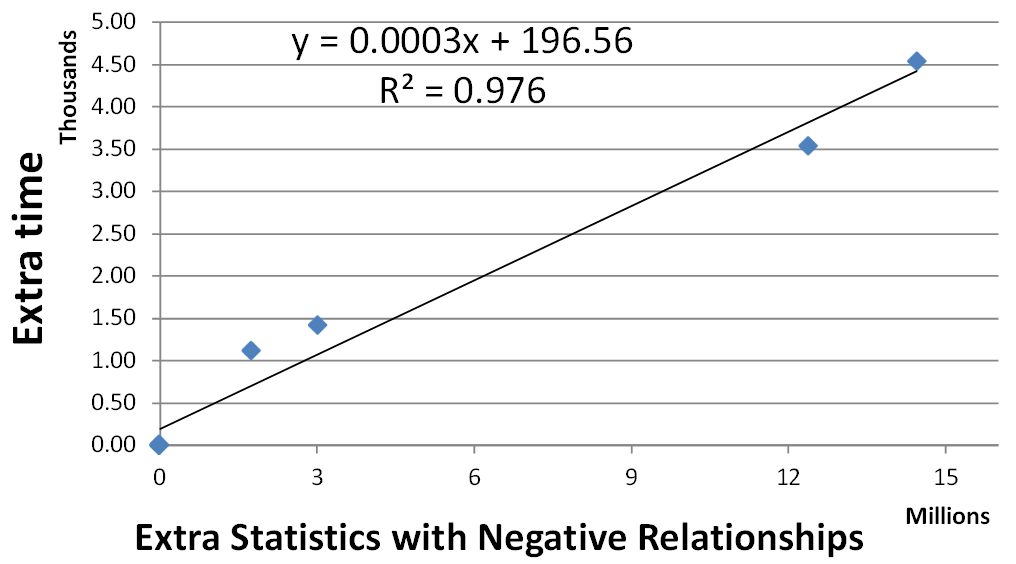}
}

\caption{M\"obius Join Extra Time (s)
\label{fig:runtime-vj}}
\end{center}
\end{figure}

\subsection{Contingency Tables with Negative Relationships vs. Positive Relationships Only} 
In this section we compare the time and space costs of computing both positive and negative relationships, vs. positive relationships only.
We use the following terminology. \textbf{Link Analysis On} refers to using a contingency table with sufficient statistics for both positive and negative relationships. 
An example is table $\ct$ in Figure~\ref{fig:flow}. 
\textbf{Link Analysis Off} refers to using a contingency table with sufficient statistics for positive relationships only. An example is table $\ct_{\true}^{+}$ 
 in Figure~\ref{fig:flow}. Table~\ref{table:link-onoff} shows the  number of sufficient statistics required for link analysis on vs. off. The difference between the link analysis on statistics  and the link analysis off statistics is the number of Extra Statistics.
The Extra Time column shows how much time the \MJ algorithm requires to compute the Extra Statistics {\em after} the contingency tables for positive relationships are constructed using SQL joins. As Figure~\ref{fig:runtime-vj} illustrates, the Extra Time stands in a nearly linear relationship to the number of Extra Statistics, which confirms the analysis of Section~\ref{sec:complexity}. Figure~\ref{fig:breakdown-vj} shows that most of the \MJ run time is spent on the Pivot component (Algorithm~\ref{alg:pivot}) rather than the main loop (Algorithm~\ref{alg:fmt}). In terms of $\ct$-table operations, most time is spent on subtraction/union rather than cross product.

\begin{figure}[htbp]
\begin{center}
\resizebox{0.4\textwidth}{!}{
\includegraphics[width=0.6\textwidth]{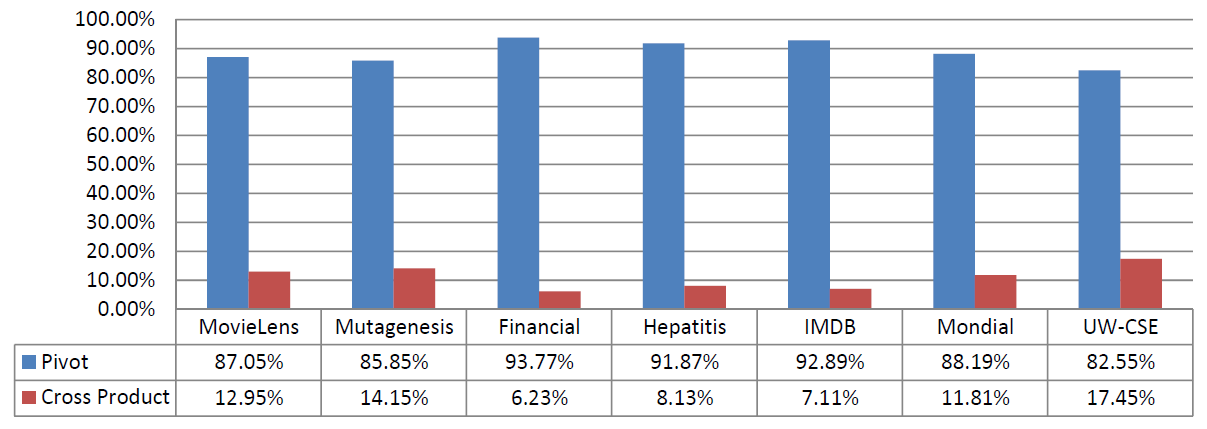}
}

\caption{Breakdown of \MJ Total Running Time
\label{fig:breakdown-vj}}
\end{center}
\end{figure}

\section{Statistical Applications}
We evaluate using link analysis on three different types of cross-table statistical analysis: feature selection, association rule  mining, and learning a Bayesian network.

\subsection{Feature Selection} For each database, we selected a target for classification, then used Weka's CFS feature subset selection method (Version 3.6.7) to select features for classification \cite{Hall2009}, given a contingency table. The idea is that if the existence of relationships is relevant to classification, then there should be a difference between the set selected with link analysis on and that selected with link analysis off. 
We measure how different two feature sets are by 1-Jaccard's coefficient:
$$\it{Distinctness}(A,B) = 1- \frac{A \cap B}{A \cup B}.$$

\begin{table}[htbp] \centering
\resizebox{0.5\textwidth}{!}{
\begin{tabular}{|l|R{2.2cm}|R{1.8cm}|R{2cm}|r|} \hline
{\multirow{2}[4]{*}{Dataset}} &{\multirow{2}[4]{*}{Target variable}} & \multicolumn{2}{c|}{\# Selected Attributes} & {\multirow{2}[4]{*}{Distinctness}}\\ \cline{3-4} 
 & & Link Analysis Off & Link Analysis On / Rvars & \\\hline
MovieLens & Horror(M) & 2 & 2 / 0 &  0.0 \\\hline
Mutagenesis & inda(M) & 3 & 3 / 0 & 0.0 \\\hline
Financial & balance(T) & 3 & 2 / 1 & 1.0 \\\hline
Hepatitis & sex(D) & 1 & 2 / 1 & 0.5 \\\hline
IMDB & avg\_revenue(D) & 5 & 2 / 1 & 1.0 \\\hline
Mondial & percentage(C) & Empty CT  & 4 / 0 & 1.0 \\\hline
UW-CSE & courseLevel(C) & 1 & 4 / 2 & 1.0 \\\hline
\end{tabular}
}
\caption{Selected Features for Target variables for  Link Analysis Off vs. Link Analysis On. Rvars denotes the number of relationship features selected. 
\label{table:feature-select}}
\end{table}

Distinctness measures how different the selected feature subset is with link analysis on and off, on a scale from 0 to 1. Here 1 = maximum dissimilarity.
Table~\ref{table:feature-select} compares the feature sets selected. In almost all datasets, sufficient statistics about negative relationships generate new relevant features for classification. 
In 4/7 datasets, the feature sets are disjoint (coefficient = 1). For the Mutagenesis and MovieLens data sets, no new features are selected. 

While Table~\ref{table:feature-select} provides evidence that relationship features are relevant to the class label, it is not straightforward to evaluate their usefulness by adding them to a relational classifier. The reason for this is that 
relational classification requires some kind of mechanism for aggregating/combining information from a target entity's relational neighborhood. There is no standard method for performing this aggregation \cite{Dzeroski2001c}, so one needs to study the interaction of the aggregation mechanism with relationship features. We leave for future work experiments that utilize relationship features in combination with different relational classifiers.

\subsection{Association Rules} A widely studied task is finding interesting association rules in a database. We considered association rules of the form $\it{body} \rightarrow \it{head}$, where $\it{body}$ and $\it{head}$ are conjunctive queries. An example of a cross-table association rule for Financial is 
$$\it{statement\_freq.(Acc)} = \it{monthly} \rightarrow \it{HasLoan}(\it{Acc},\it{Loan}) = \true.$$
We searched for interesting rules using both the link analysis off and the link analysis on contingency tables for each database. The idea is that if a relationship variable is relevant for other features, it should appear in an association rule. With link analysis off, all relationship variables always have the value $\true$, so they do not appear in any association rule. We used Weka's Apriori implementation to search for association rules in both modes. The interestingness metric was Lift. Parameters were set to their default values. Table~\ref{table:association} shows the number of rules that utilize relationship variables with link analysis on, out of the top 20 rules. In all cases, a majority of rules utilize relationship variables,  in Mutagenesis and IMDB all of them do. 
\begin{table}[htbp] \centering
\resizebox{0.5 \textwidth}{!}{
\begin{tabular}{|c|c|c|c|c|c|c|c|}
\hline
\multicolumn{1}{|c|}{Dataset} & MovieLens & Mutagenesis & Financial & Hepatitis & IMDB  & Mondial & UW-CSE \\
\hline
\# rules  & 14/20  & 20/20 & 12/20 & 15/20 & 20/20 & 16/20 & 12/20 \\ 
\hline 
\end{tabular}%
}
\caption{Number of top 20 Association Rules that utilize relationship variables.}
  \label{table:association}%
\end{table}%

\subsection{Learning Bayesian Networks}

Our most challenging application is constructing a Bayesian network (BN) for a relational database. For single-table data, Bayesian network learning has been considered as a benchmark application for precomputing sufficient statistics \cite{Moore1998,lv2012}. A Bayesian network structure is a directly acyclic graph whose nodes are random variables. Given an assignment of values to its parameters, a Bayesian network represents a joint distribution over both attributes and relationships in a relational database. Several researchers have noted the usefulness of constructing a graphical statistical model for a relational database ~\cite{Graepel_CIKM13,Wang2008}.
For data exploration, a Bayes net  model provides a succinct graphical representation of complex statistical-relational correlations. The model also supports probabilistic reasoning for answering ``what-if'' queries about the probabilities of uncertain outcomes. 

We used the previously existing learn-and-join method (LAJ), which is the state of the art for Bayes net learning in relational databases \cite{Schulte2012}. The LAJ method takes as input a contingency table for the entire database, so we can apply it with both link analysis on and link analysis off to obtain two different BN structures for each database. Our experiment is the first evaluation of the LAJ method with link analysis on. We used the LAJ implementation provided by its creators.
We score all learned graph structures using the same full contingency table with link analysis on, so that the scores are comparable. The idea is that turning link analysis on should lead to a different structure that represents correlations, involving relationship variables, that exist in the data.

\subsubsection{Structure Learning Times} 
Table~\ref{table:runtimes} provides the model search time for structure learning with link analysis on and off. Structure learning is fast, even for the largest contingency table  IMDB (less than 10 minutes run-time). With link analysis on, structure learning takes more time as it processes more information. 
In both modes, the run-time for building the contingency tables (Table~\ref{table:cttimes}) dominates the structure learning cost. For the Mondial database, there is no case where all relationship variables are simultaneously true, so with link analysis off the contingency table is empty.

\begin{table} \centering
\resizebox{2.5in}{!}{
\begin{tabular}[c]
{|l|R{3cm}|R{3cm}|}\hline
 \textbf{Dataset}  & \textbf{Link Analysis On } & \textbf{Link Analysis Off } \\\hline
Movielens & 1.53&1.44 \\\hline
Mutagenesis & 1.78&1.96 \\\hline
Financial  &96.31& 3.19 \\\hline
Hepatitis   & 416.70& 3.49\\\hline
IMDB   & 551.64 & 26.16 \\\hline
Mondial & 190.16&N/A\\\hline
UW-CSE & 2.89&2.47 \\\hline
\end{tabular}
} 
\caption{Model Structure Learning Time  in seconds.  
 \label{table:runtimes}}
\end{table}

\subsubsection{Statistical Scores.}
We report two model metrics, the log-likelihood score, and the model complexity as measured by the number of parameters. The \textbf{log-likelihood} is denoted as $L(\hat{G},\d)$, where $\hat{G}$ is the BN $\G$ with its parameters instantiated to be the maximum likelihood estimates given the dataset $\d$, and the quantity $L(\hat{G},\d)$ is the log-likelihood of $\hat{G}$ on $\d$. 
We use the relational log-likelihood score defined in \cite{Schulte2011}, 
which differs from the standard single-table Bayes net  likelihood 
only by replacing counts by frequencies  so that scores are comparable across different nodes and databases. 
To provide information about the qualitative graph structure learned, we report edges learned that point to a relationship variable as a child. Such edges can be learned only with link analysis on. We distinguish edges that link relationship variables---R2R---and that link attribute variables to relationships---A2R.

\begin{table}[htb] 
\begin{center}
\resizebox{0.5 \textwidth}{!}{
\begin{tabular}{|m{2.5cm}|R{2cm}|R{1.8cm}|R{1cm}|R{1cm}| }
\hline \textbf{Movielens } &{log-likelihood} &{\#Parameter}& {R2R}&{A2R}\\
\hline
Link Analysis Off    & -4.68 & \textbf{164} & 0     & 0 \\
\hline
Link Analysis On  & \textbf{-3.44} & 292   & 0     & 3 \\
\hline
		\end{tabular}
}
\end{center}

\begin{center}
\resizebox{0.5 \textwidth}{!}{
\begin{tabular}{|m{2.5cm}|R{2cm}|R{1.8cm}|R{1cm}|R{1cm}|  }\hline 
\textbf{Mutagenesis } &{log-likelihood} &{\#Parameter}& {R2R}&{A2R}\\	 \hline 
Link Analysis  Off   & -6.18 & \textbf{499} & 0     & 0 \\ \hline
Link Analysis  On  & \textbf{-5.96} & 721   & 1     & 5  \\ \hline
		\end{tabular}
}
\end{center}
\begin{center}
\resizebox{0.5 \textwidth}{!}{
\begin{tabular}{|m{2.5cm}|R{2cm}|R{1.8cm}|R{1cm}|R{1cm}|  }
\hline \textbf{Financial } &{log-likelihood} &{\#Parameter}& {R2R}&{A2R}\\	  \hline
Link Analysis Off   & -10.96 & 11,572 & 0     & 0 \\ \hline
Link Analysis On  & \textbf{-10.74} & \textbf{2433} & 2     & 9 \\ \hline
		\end{tabular}
}
\end{center}	

\begin{center}
\resizebox{0.5 \textwidth}{!}{
\begin{tabular}{|m{2.5cm}|R{2cm}|R{1.8cm}|R{1cm}|R{1cm}|  }
\hline \textbf{Hepatitis  } &{log-likelihood} &{\#Parameter}& {R2R}&{A2R}\\	  \hline
Link Analysis Off   & \textbf{-15.61} & 962   & 0     & 0 \\ 			\hline
Link Analysis On  & -16.58 & \textbf{569} & 3     & 6 \\ 			\hline
		\end{tabular}
}
\end{center}

\begin{center}
\resizebox{0.5 \textwidth}{!}{
\begin{tabular}{|m{2.5cm}|R{2cm}|R{1.8cm}|R{1cm}|R{1cm}|  }
\hline \textbf{IMDB  } &{log-likelihood} &{\#Parameter}& {R2R}&{A2R}\\	  \hline
Link Analysis Off    & -13.63 & 181,896 & 0     & 0 \\ 		\hline
Link Analysis On  & \textbf{-11.39} & \textbf{60,059} & 0     & 11 \\ 		\hline
		\end{tabular}
}
\end{center}

\begin{center}
\resizebox{0.5 \textwidth}{!}{
\begin{tabular}{|m{2.5cm}|R{2cm}|R{1.8cm}|R{1cm}|R{1cm}|  }
\hline \textbf{Mondial  } &{log-likelihood} &{\#Parameter}& {R2R}&{A2R}\\	  \hline
Link Analysis Off   &   N/A    &  N/A     & N/A     & N/A \\ 		\hline
Link Analysis On& -18.2 & 339   & 0     & 4 \\ 		\hline
		\end{tabular}
}
\end{center}

\begin{center}
\resizebox{0.5 \textwidth}{!}{
\begin{tabular}{|m{2.5cm}|R{2cm}|R{1.8cm}|R{1cm}|R{1cm}|  }
\hline \textbf{UW-CSE   } &{log-likelihood} &{\#Parameter}& {R2R}&{A2R}\\	  \hline
Link Analysis Off   & \textbf{-6.68} & 305   & 0     & 0 \\  \hline
Link Analysis On & -8.13 & \textbf{241} & 0     & 2 \\  \hline
		\end{tabular}
}
\end{center}
\caption{Comparison of Statistical Performance of Bayesian Network Learning.
}
\label{table:result_scores}
\end{table}

Structure learning can use the new type of dependencies to find a better, or at least different, trade-off between model complexity and model fit.
On two datasets (IMDB and Financial), link analysis leads to a superior model that achieves better data fit with fewer parameters. These are also the datasets with the most complex relational schemas (see Table~\ref{table:datasetsize}). On IMDB in particular, considering only positive links leads to a very poor structure with a huge number of parameters.
On four datasets, extra sufficient statistics lead to different trade-offs: On MovieLens and Mutagenesis, link analysis leads to better data fit but higher model complexity, and the reverse for Hepatitis and UW-CSE.

\section{Related Work} 

{\em Sufficient Statistics for Single Data Tables.} Several data structures have been proposed for storing sufficient statistics defined on a {\em single} data table. 
One of the best-known are ADtrees \cite{Moore1998}. 
An ADtree provides a memory-efficient data structure for {\em storing} and retrieving sufficient statistics once they have been computed. 
In this paper, we focus on the problem of {\em computing} the sufficient statistics, especially for the case where the relevant rows have not been materialized. 
Thus ADtrees and contingency tables are complementary representations for different purposes: contingency tables support a computationally efficient block access to sufficient statistics, whereas ADtrees provide a memory efficient compression of the sufficient statistics. 
An interesting direction for future work is to build an ADtree for the contingency table once it has been computed. 

{\em Relational Sufficient Statistics.} 
Schulte {\em et al.} review previous methods for computing statistics with negative relationships \cite{Schulte2014}. They show that the fast M\"obius transform can be used in the case of multiple negative relationships. 
Their evaluation considered only Bayes net parameter learning with only one relationship. 
We examined computing joint sufficient statistics over the entire database. 
Other novel aspects are the $\ct$-table operations and using the relationship chain lattice to facilitate dynamic programming.

\section{Conclusion} \label{sec:conclusion} 
Utilizing the information in a relational database for statistical modelling and pattern mining requires fast access to multi-relational sufficient statistics, that combine information across database tables. 
We presented an efficient dynamic program that computes sufficient statistics for any combination of positive {\em and} negative relationships, starting with a set of statistics for positive relationships only.
Our dynamic program performs a virtual join operation, that counts the number of statistics in a table join without actually constructing the join. We showed that the run time of the algorithm is $O(r \log r)$, where $r$ is the number of sufficient statistics to be computed.
The computed statistics are stored in contingency tables.
We introduced contingency table algebra, an extension of relational algebra, to elegantly describe and efficiently implement the dynamic program. 
Empirical evaluation on seven benchmark databases demonstrated the scalability of our algorithm; we compute sufficient statistics with positive and negative relationships in databases with over 1 million data records.  
Our experiments illustrated how access to sufficient statistics for both positive and negative relationships enhances feature selection, rule mining, and Bayesian network learning.

\emph{Limitations and Future Work.} 
Our dynamic program scales well with the number of rows, but not with the number of columns and relationships in the database. 
This limitation stems from the fact that the contingency table size grows exponentially with the number of random variables in the table. In this paper, we applied the algorithm to construct a large table for {\em all} variables in the database. We emphasize that this is only one way to apply the algorithm. The M\"obius Join algorithm efficiently finds cross-table statistics for any set of variables, not only for the complete set of all variables in the database. An alternative is to apply the virtual join only up to a prespecified relatively small relationship chain length.
Another possibility is to use postcounting \cite{lv2012}: Rather than precompute a large contingency table prior to learning, compute many small contingency tables for  small subsets of variables on demand during learning. 
%
%


In sum, our M\"obius Virtual Join algorithm efficiently computes query counts which may involve any number of {\em positive and negative }relationships. 
These sufficient statistics support a scalable statistical analysis of  associations among both relationships and attributes in a relational database.

\section*{Acknowledgments}
This research was supported by a Discovery grant to Oliver Schulte by the Natural Sciences and Engineering Research Council of Canada. 
Zhensong Qian was  supported by a grant from the China Scholarship Council.

\bibliographystyle{abbrv}
\bibliography{master} 

\balance

\end{document}

%% file: preamble-stuff.tex
\newtheorem{proposition}{Proposition}

\usepackage[ruled,vlined]{algorithm2e}
\usepackage{algorithmic}
\usepackage{amsfonts}
\usepackage{amssymb}
\usepackage{graphicx}
\usepackage{url}
\usepackage{subfigure}
\usepackage{epstopdf}
\usepackage{multirow}
\usepackage{subfigure}
\usepackage{ifthen}

\renewcommand{\d}{\mathbf{d}}

\newcommand{\V}{V}

\newcommand{\A}{\mathbb{A}}

\newcommand{\C}{\mathbb{C}}
\renewcommand{\P}{\mathbb{P}}
\newcommand{\R}{\mathbb{R}}

\newcommand{\Y}{\mathbb{Y}}

\newcommand{\x}{x}

\newcommand{\TTuple}[1][0.0ex]{\vec{t}\hspace{#1}}

\newcommand{\UTuple}[1][0.0ex]{\vec{u}\hspace{#1}}

\newcommand{\VTuple}{\vec{v}}

\newcommand{\Mrange}[1]{\ifthenelse{\equal{#1}{T}}{\TTuple_m}{\ifthenelse{\equal{#1}{U}}{\UTuple_m}{\ifthenelse{\equal{#1}{V}}{\VTuple_m}{\mbox{UNKNOWN
TERM ID}}}}}
\newcommand{\Prange}[1]{\ifthenelse{\equal{#1}{T}}{\vec{t}_{pa}}{\ifthenelse{\equal{#1}{U}}{\vec{u}_{pa}}{\ifthenelse{\equal{#1}{V}}{\vec{v}_{pa}}{\mbox{UNKNOWN
TERM ID}}}}}

\newcommand{\GroundPrange}[1]{\ifthenelse{\equal{#1}{T}}{\vec{t}_{pa,\grounding'}}{\ifthenelse{\equal{#1}{U}}{\vec{u}_{pa,\grounding'}}{\ifthenelse{\equal{#1}{V}}{\vec{v}_{pa,\grounding'}}{\mbox{UNKNOWN
TERM ID}}}}}

\newcommand{\qcount}{\it{count}}

\newcommand{\project}{\pi}
\newcommand{\select}{\sigma}
\newcommand{\condition}{\chi}
\newcommand{\selectcond}{\phi}

\newcommand{\MJ}{MJ\:}

\newcommand{\population}{\mathcal{P}}

\newcommand{\outdomain}{V}

\newcommand{\argterms}{\bs{\tau}}

\newcommand{\grounding}{\gamma}

\newcommand{\row}{r}
\newcommand{\statistics}{s}

\newcommand{\Relation}{R} 

\newcommand{\eatts}{\it{1Atts}}
\newcommand{\eatt}{\it{1Att}}
\newcommand{\ratts}{\it{2Atts}}

\newcommand{\atts}{\it{Atts}}
\newcommand{\Nodes}{\it{Vars}}

\newcommand{\ColumnList}{\it{CL}}

\newcommand{\G}{G}

\newcommand{\X}{\mathbb{X}}

\renewcommand{\R}{R} 

\newcommand{\D}{\mathcal{D}} 
\renewcommand{\S}{\mathbb{S}} 

\newcommand{\student}{\mathit{Student}}

\newcommand{\reg}{\mathit{Registration}}
\newcommand{\ra}{\mathit{RA}}

\newcommand{\true}{\mathrm{T}}
\newcommand{\false}{\mathrm{F}}
\newcommand{\functor}{f}

\def\set#1{\mathbf{#1}}
\def\bs#1{\boldsymbol{#1}}